\documentclass[10pt,conference]{IEEEtran}
\IEEEoverridecommandlockouts
\usepackage{cite}
\usepackage{amssymb,amsfonts}
\usepackage{algpseudocode}
\usepackage{graphicx}
\usepackage{textcomp}
\usepackage[tbtags]{amsmath}
\usepackage{diagbox}
\usepackage{color}
\usepackage{amsthm} 
\usepackage{algorithm}
\usepackage{algorithmicx}
\usepackage{booktabs}
\usepackage{amsmath}
\usepackage{bbm}
\usepackage{subfig}
\newtheorem{theorem}{Theorem}
\usepackage{graphicx}
\usepackage{epsfig}
\usepackage{epstopdf}
\usepackage{multirow}
\usepackage{caption}
\usepackage{marvosym}
\usepackage[thicklines]{cancel}
\usepackage[left=0.64in,right=0.68in,bottom=1.04in,top=0.71in]{geometry}%0.98 bottom 0.59 right

\def\BibTeX{{\rm B\kern-.05em{\sc i\kern-.025em b}\kern-.08em
    T\kern-.1667em\lower.7ex\hbox{E}\kern-.125emX}}

\newtheorem{lemma}{Lemma}

\newtheorem{Assumption}{Assumption}
\begin{document}

\title{Mobility-Aware Federated Learning: Multi-Armed Bandit Based Selection in Vehicular Network}

\author{
	\IEEEauthorblockN{Haoyu Tu\IEEEauthorrefmark{1}\IEEEauthorrefmark{2},  Lin Chen\IEEEauthorrefmark{2},  Zuguang Li\IEEEauthorrefmark{1}, Xiaopei Chen\IEEEauthorrefmark{1}, and Wen Wu\textsuperscript{\Letter}\IEEEauthorrefmark{1}\\}
	\IEEEauthorblockA{\IEEEauthorrefmark{1}Frontier Research Center, Pengcheng Laboratory, Shenzhen 518000, China\\}
	\IEEEauthorblockA{\IEEEauthorrefmark{2}School of Computer Science and Engineering, Sun Yat-sen University, Guangzhou 510006, China\\
	%\IEEEauthorblockA{\IEEEauthorrefmark{4}School of Future Technology, South China University of Technology, Guangzhou 510641, China\\
		\{tuhy, lizg01, chenxp, wuw02\}@pcl.ac.cn, chenlin69@mail.sysu.edu.cn}
\thanks{\textsuperscript{\Letter}\IEEEauthorrefmark{1}Wen Wu is the corresponding author of this paper.}}

%}

\maketitle

%20241004 rewise的暂时都没改 20241008 revise的有一部分改了。。。

\begin{abstract}
In this paper, we study a vehicle selection problem for federated learning (FL) over vehicular networks. Specifically, we design a mobility-aware vehicular federated learning (MAVFL)  scheme in which vehicles drive through a road segment to perform FL. Some vehicles may drive out of the segment which leads to unsuccessful training. In the proposed scheme, the real-time successful training participation ratio is utilized to implement vehicle selection. We conduct the convergence analysis to indicate the influence of vehicle mobility on training loss. Furthermore, we propose a multi-armed bandit-based vehicle selection algorithm to minimize the utility function considering training loss and delay. The simulation results show that compared with baselines,  the proposed algorithm can achieve better training performance with approximately 28\% faster convergence.
\end{abstract}

%\begin{IEEEkeywords}
%	component, formatting, style, styling, insert
%\end{IEEEkeywords}

%\addtolength{\hoffset}{-0.08in}

%20240821 改成需要各种业务，数据分布的，。。两个通信不用提。。。。however尽量不用，改成while....
\section{Introduction}
Data-driven machine learning (ML) tasks in vehicular networks such as trajectory prediction, object detection and traffic sign classification enhance road safety and alleviate urban congestion to facilitate autonomous driving \cite{Shen2022}. The distributed data of each vehicle is collected by various sensors such as GPS (Global Positioning System), LiDAR (Light Detection and Ranging) and cameras, and increased data privacy and communication overhead is brought in when local data is offloaded to the server. Federated learning (FL) enables vehicles to collaboratively train models from the server aggregated from all vehicles without sharing local data directly and reduces the communication overhead caused by large amounts of data transmission between vehicles to the server and cloud \cite{Li2021,Wu2022}. Nevertheless, vehicle mobility brings issues for FL in vehicular networks with dynamic communication channels and time-varying available vehicle set~\cite{Posner2021}. 

%20240927增加论文引用，tbd

%20240821 in the literature......FL....

In the literature, the research to speed up FL convergence in vehicular networks can be roughly divided into two categories: model aggregation design \cite{Mob01,Li2022} and resource allocation \cite{Zeng2022,Xie2023} . In terms of model aggregation, the successful probability of vehicle training is optimized by designing a weighted parameter with the duration time and size of the dataset for each vehicle \cite{Mob01}. The inner-cluster and inter-cluster training scheme is proposed for vehicles with multi-hop clusters \cite{Li2022}. An incentive mechanism based on contract theory is proposed to select vehicles with better quality wireless channels and the dataset size \cite{Zeng2022}. Xie \emph{et al.} in \cite{Xie2023} optimized the round duration and local iteration number to measure the performance of FL in vehicular networks considering mobility. Different from existing papers, we study the impact of moving vehicles and develop an online algorithm to select vehicles based on their locations.

Designing an efficient FL scheme in vehicular networks faces the following challenges. Firstly,  the designed model aggregation algorithms require prior knowledge of the vehicle’s mobility which may not be easy to get such as the historical trace information in the current area. Secondly, the modeling of mobility involves viewing movement as a given probability, but the probability is statistical in long-time period and difficult to provide \emph{real-time} information to guide vehicle selection and resource allocation choices.

In this paper, we design a \underline{M}obility-\underline{A}ware \underline{V}ehicular \underline{F}ederated \underline{L}earning (MAVFL) scheme, where vehicles drive through the road segment to participate in FL with a collaborative base station~(BS).  We propose the real-time ratio that vehicles successfully upload models. We conduct the theoretical analysis of convergence and demonstrate that the ratio significantly influences convergence. Based on analytical results, we formulate the optimization problem to maximize the utility function while minimizing training loss and training delay. We design an MAB-based vehicle selection algorithm to solve the optimization problem. Extensive simulation results show the effectiveness of the proposed scheme in terms of convergence speed and training delay.%We consider the scheme as CAVs moving on the road, and only those within range of the base station can proceed training. During the process of model training, some vehicles may have left the segment due to the mobility of vehicle movement, which makes them unable to continue training.
%During the process of model training, some CAVs may have left the segment due to the mobility of vehicle movement, which makes them unable to continue training. Different from above papers, we focus on the influence of  moving out vehicles and design online algorithm to implement  vehicle selection based on real-time information. The proposed scheme can achieve better performance both on convergence and training delay.
%20240808实时性的体现。。。

The main contributions of our paper are summarized as follows:

\begin{itemize}
\item  We propose an MAVFL scheme and conduct the convergence proof of the proposed scheme.% and minimize training delay.%considering real-time indicator as the proportion of vehicles successfully participating in training. 

\item  We formulate an optimization problem to speed up the convergence.  We propose an MAB-based vehicle selection algorithm to solve the problem.  %balance the exploration and exploitation. 
%through choosing appropriate set of vehicles
%\item We conduct simulation of proposed MAB-based vehicle selection algorithm and the simulation results show the effective of proposed algorithm both in delay and convergence speed.
\end{itemize}

%The remainder of this paper is organized as follows. In Section~\ref{2}, we introduce the system model of our MAVFL scheme. The convergence analysis is given in Section~\ref{3}. The optimization problem is formulated in Section~\ref{4}. We develop the MAB-based vehicle selection algorithm in Section~\ref{5}. The simulation results are given in Section~\ref{6}. Section~\ref{7} concludes the paper. 
\begin{figure}
	\centering
    \includegraphics[width=0.5\textwidth]{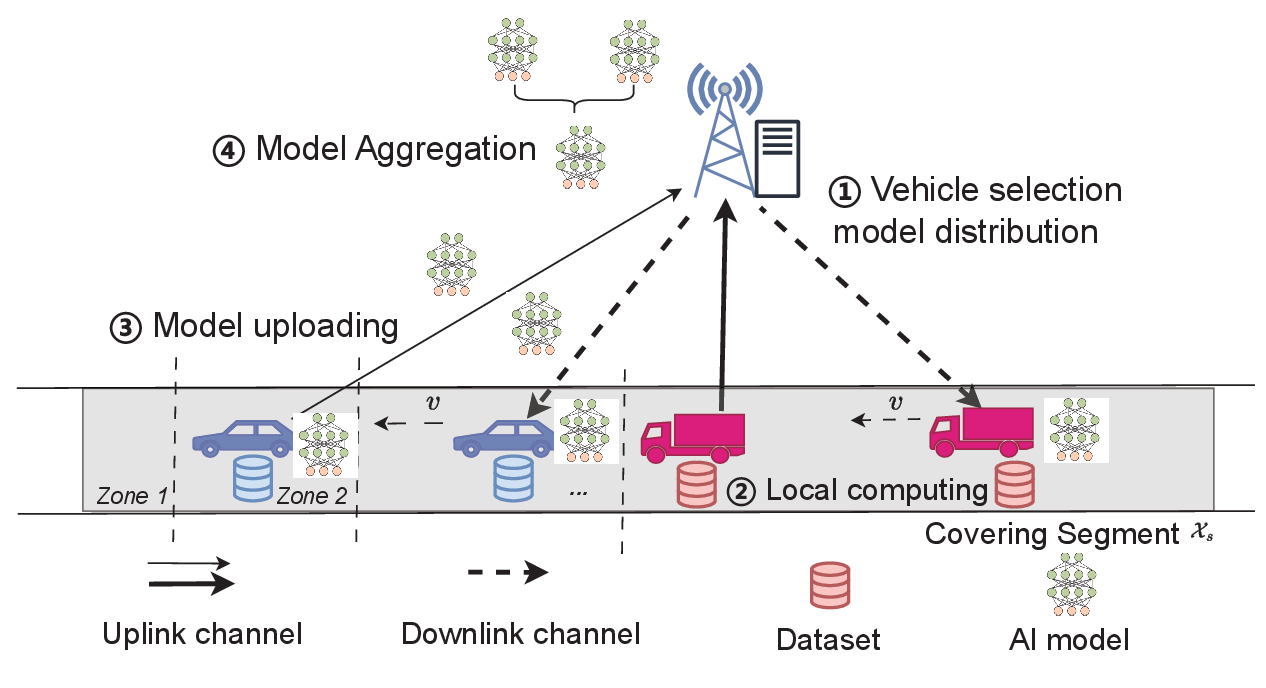}
	\caption{System Model}
	\label{fig}
\end{figure}

%20241004第二部分回头自己重新看

%\section{Related Work}

\section{Considered Scenario and System Model}\label{2}
 %We also give the communication and computation model of vehicles for our scheme. 

%20240511Motivated by....[] Mob-FL....前面这些不一定要在这里写。。。
%20240511
\subsection{Considered Scenario}
 As shown in Fig. \ref{fig}, we consider a segment of road covered by a BS with a collaborative server and $K_0$ vehicles that arrive at this stretch of road over a  period of time. These vehicles will communicate with the BS in the segment of road to train the model through FL paradigm, and each vehicle $k\in\{1,2,\ldots,K_0\}$ has collected data samples as $D_k$. {Here, the data is highly pertinent to vehicular services
 such as image recognition and trajectory prediction, which has
 been collected and well-labeled before training.}  The length of covering segment is $L$, and the segment is divided into multiple zones denoted as $\mathcal{Z}=\{1,2,\ldots,Z\}$. As such, the distance between the vehicle and the BS in the same zone $z\in\mathcal{Z}$ is considered as being the same.  %As such in same zone $z\in\mathcal{Z}$ the vehicles have the same distance between BS.  
%Subsequently, 
\subsection{Vehicular Federated Learning Scheme}
 In the proposed scheme, the training process contains $R$ rounds before the deadline $T_d$. The server is able to get the location and velocity of vehicle $k$ within the segment during training as $x_k^t$ and $v_k^t$ at $t_0<T_d$.

During the training process of the MAVFL scheme, the loss function for vehicle $k$ is expressed as $$f(\mathbf{w}_k,D_k)=\frac1{|\mathcal{D}_k|}\sum_{i\in \mathcal{D}_k}\ell(\mathbf{w}_k;z^i_k),$$
%\begin{equation}
%	\begin{aligned}
%		f(\mathbf{w}_k,D_k)& =\frac1{|\mathcal{D}_k|}\sum_{i\in {D}_k|}\ell(\mathbf{w}_k;z^i_k),  
%	\end{aligned}
%\end{equation}
where $\mathcal{D}_k$ means the indices of samples in ${D}_k$, $z^i_k$ is the data sample representation for sample $i\in\{1,2,\ldots,|\mathcal{D}_k|\}$ and $\ell$ is the local loss function. Vehicle $k$ will train model $\mathbf{w}$ to find the optimal model as 
$\hat{\mathbf{w}_{k}} =\arg\min_{\mathbf{w}_k}\frac1{|\mathcal{D}_k|}\sum_{i=1}^{|\mathcal{D}_k|}\ell(\mathbf{w}_k;z^i_k),$ and the global loss function  is $F(w)=\sum_{k=1}^K q_k f(\mathbf{w}_k,D_k),$ where $q_k$ is the aggregation weight of vehicle $k$. %Next we introduce the procedures of federated learning influenced by mobility of vehicle. 
%20240822增加车辆选择的部分
\subsubsection{Vehicle Selection and Model Distribution}
At the beginning of round $r\in\{1,2,\ldots,R\}$, the server owns the newest global model $w^{(r)}$, then it will make the vehicle selection decision to generate the set of vehicles $\mathcal{S}^{r}$ and distribute the model to vehicles within its covering segment as 
\begin{equation}
	\mathbf{w}^{r,0}_k \leftarrow \mathbf{w}^r,   
	\forall k, x_k^{r,0}\in  \mathcal{X}_s, k \in\mathcal{S}^{r},
\end{equation}
%20240511下午继续写。。。
where $\mathbf{w}^{r,0}_k$ means the local model for vehicle $k$ in round $r$ and epoch 0, $\mathcal{X}_s$ is the location range of covering segment, and $x_k^{r,0}$ is the location of vehicle $k$ at the beginning of round~$r$. %Here the set of CVs receiving global model in round $r$ is got as $\mathcal{S}^{r}$ for vehicle $k \in \mathcal{S}^{r}$. 

\subsubsection{Local Updating}
After receiving model $w^{r,0}_k$, vehicle $k$ will perform local stochastic gradient descent (SGD) for $E$ epochs. For local training epoch $e\in\{0,\ldots,E-1\}$, the gradient descent is expressed as $g_{k}^{r,e}\leftarrow\frac1{|\mathcal{D}_k|}\sum_{i=1}^{|\mathcal{D}_k|}\nabla\ell(\mathbf{w}_k^{r,e};z^i_k)$ and vehicle $k$ performs local SGD as
\begin{equation}
	\mathbf{w}_k^{r,e+1}\leftarrow \mathbf{w}_k^{r,e}-\eta {g}_{k}^{r,e}, \text{for }e\in[0,E-1].  
\end{equation}

\subsubsection{Model Uploading}
After local updating, vehicle $k$ will try to upload model updates $g^{r}_k=\sum_{e=0}^{E-1} g_k^{r,e}$ to the server. Due to the movement of vehicles, it may drive out of the covering segment of the BS, which causes the model missing. We define the dropout indicator $\mathbbm{1}_k^{r}\in\{0,1\}$ as
\begin{equation}\nonumber
	\mathbbm{1}_k^{r}=\left\{
	\begin{aligned}
		&1, \quad &x_k^{r,E}\in\mathcal{X}_s,\\
		%&\sqrt{(x_k-L/2)^2+H^2} \quad &\text{case 2} \\
		&0, \quad &x_k^{r,E}\notin\mathcal{X}_s
	\end{aligned}
	\right.
\end{equation}
to represent the state of whether vehicle $k$ stays in the covering segment after local updating with $E$ epochs in round $r$. 

   %To handle this problem, the model transferring step is brought in.   

%\subsubsection{Model uploading}
%In this stage, each vehicle $k$ will transfer its local model $\mathbf{w}_k^{(r,E)}$ to near vehicles as $j \in \mathcal{H}_k^{(r)}$. After model transferring, each vehicle $k$ will receive models as 
%$$w_{j}^{(r,E)}, \forall j \in \mathcal{H}_k^{(r)}.$$ Here $\mathcal{H}_k^{(r)}$ is the set of vehicles within the V2V communication range of vehicle $k$ including vehicle $k$ itself. Each vehicle $k$ is dedicated to aggregate received models to generate new model as 
%\begin{equation}
%	w_{\text{a},k}^{(r)}=\sum_{j\in \mathcal{H}_k^{(r)}}w_{j}^{(r,E)}/|\mathcal{H}_k^{(r)}|.
%\end{equation}
%Then a subset of head vehicles within sets $\mathcal{H}_k^{(r)}, k\in\mathcal{S}^{(r)}$ and within covering set is chosen  as $\mathcal{H}^{(r)}$ to upload aggregated models.

\subsubsection{Model Aggregation}
The server will wait for a period of time to receive model updates from vehicles within the covering segment. If the server receives any model updates within $T_{\text{max}}$, it will aggregate these models as 
\begin{equation}
	%w^{(r+1)}=\sum_{k}\mathbbm{1}_k^{(r)} \mathbf{w}_k^{(r,E)}/(\sum_k \mathbbm{1}_k^{(r)})
	\mathbf{w}^{r+1}=\mathbf{w}^r-\eta\sum_{k}\mathbbm{1}_k^{r}(\frac{g_k^{r}}{\sum_{k} {\mathbbm{1}_k^{r}}}), \forall \sum_{k} {\mathbbm{1}_k^{r}}\neq 0%=\mathbf{w}^r-\eta^{r}G^r
\end{equation}
to get new global model $\mathbf{w}^{r+1}$. For each global model $\mathbf{w}^{r+1}$, we define the set of vehicles as $\mathcal{N}^{r}$ which denotes the set of vehicles uploading models included in the global model as 
\begin{equation}
	\mathcal{N}^{r}=\mathop{\bigcup}_{k,\forall \mathbbm{1}_k^{r}=1}\{k\},
\end{equation}
then the successful training ratio $p^{r}$  is denoted as
\begin{equation}\label{p}
	p^{r}=\frac{|\mathcal{N}^{r}|}{|\mathcal{S}^{r}|}, 
\end{equation}
where $|\mathcal{N}^{r}|$ is the number of receiving uploading models from vehicles and $|\mathcal{S}^{r}|$ is the number of vehicles receiving downloading models. The value of $p^{r}$ is related to both vehicle selection and vehicle mobility. Especially, when all vehicles are stationary, the ratio $p^{r}$ is 1, and the ratio is 0 considering all selected vehicles driving out of the covering segment during local computing. When $p^{r}$ is 0, the server will distribute the global model as $\mathbf{w}^{r+1}=\mathbf{w}^r$. 

\subsection{Training Delay model}
We give the analysis of the delay of the MAVFL scheme in this part. Considering the movement of vehicle $k$, the distance of vehicle $k$ and BS as $L_{k}^{r}$ is related to the topology structure of the road and the movement of vehicles. We define the location coordinate of vehicle $k$ as $x_k$, then the distance between vehicle $k$ and BS can be expressed as  %the length of BS covering segment is $L$, and the entrance location coordinate of the given road is $x_0=L$ with the exit location coordinate is $x_1=0$, then the distance between vehicle $k$ and BS can be expressed as 
\begin{equation}\nonumber
	L_k^r=
	\begin{aligned}
		\sqrt{(L_z)^2+H^2}, x_k^r\in\mathcal{Z}_z, \\
		%&\infty \quad &\text{otherwise},
	\end{aligned}
\end{equation}
where $H$ is the height of BS, $L_z$ is the normalized distance between vehicle and BS in zone $z$, and $\mathcal{Z}_z$ is the range of location for zone $z$.  
%\begin{equation}\nonumber
%	L_k^r=\left\{
%	\begin{aligned}
	%		&\sqrt{(L_z)^2+H^2} \quad &x_k^r\in\mathcal{Z}_z; \\
	%		%&\sqrt{(x_k-L/2)^2+H^2} \quad &\text{case 2} \\
	%		&\infty \quad &\text{otherwise},
	%	\end{aligned}
%	\right
%	.
%\end{equation}

The uplink transmission for selected vehicles is considered as an orthogonal frequency-division multiple access (OFDMA), then the uplink transmission rate of vehicle $k$ is $q_k^r=B_k\log_2\left(1+\frac{P_kh_k(L_k^r)^{-\beta}}{N_0}\right)$
where $B_k$ is the bandwidth allocated for vehicle $k$, $P_k$ is the transmission power, $h_k$ is the channel power gain for vehicle $k$ and $\beta$ is pass loss exponent. Then the uplink time for vehicle $k$ is expressed as $T_{k,\text{c}}^{r}=M/q_k^r$ with uploading model size $M$.

Next we give the expressions of computation time for vehicle $k$  as $T_{k,\text{p}}=\frac{|\mathcal{D}_{k}|g_k}{c_kf_{k}},$
where $g_k$ is the number of GPU cycles required to train one bit of data, $f_k$ is the GPU frequency and $c_k$ is the normalized parameter.  

Considering the setup of server synchronous aggregation, the duration time for the MAVFL scheme is the sum of communication time and computation time as $T=T_{\text{c}}+T_{\text{p}}=\sum_{r\in\mathcal{R}} \max_{k\in\mathcal{K}} [a_k^r(T_{k,\text{c}}^{r}+T_{k,\text{p}})]$, where $a_k^r\in\{0,1\}$ is the indicator representing whether the vehicle is unselected or selected in round $r$. For each selected vehicle, the allocated bandwidth is $B_k=B/(\sum_k a_k^r)$ with total bandwidth~$B$.
%20240723这里是GPU还是CPU呢？没想好。。暂时这么写吧。。。。如果是GPU，换成之前waste这篇文章写的。。。就是第二个公式
%$B_k=B/(p^r \sum_k a_k^r)$

%\begin{equation}
	
	%T_{cp}^r(n_k^r,k)=E(\frac{S^r\vartheta_k}{\xi_n^r})=\rho^rEM(\frac{\vartheta_k}{\xi_n^r}).
%\end{equation}

\section{Convergence Analysis}\label{3}
%\subsection{Convergence Analysis}
In this section, we give the convergence proof of the MAVFL scheme considering the influence of vehicle mobility. Given the convergence, we assume the below assumptions and lemma are satisfied \cite{Mob02,Li2020}.%, which are also used and proved in previous works as \cite{Mob02,Li2020}.

%20240708之前纸上的是分层的FL。。。。

%20240708参考mobility-clustered...
\begin{Assumption}
The loss function $F$ are L-smooth as $\|\nabla F(\mathbf{w}_1) -\nabla F(\mathbf{w}_2)\|\leq L\|\mathbf{w}_1-\mathbf{w}_2\|$
	 for given $\mathbf{w}_1$ and~$\mathbf{w}_2$. 
	 %$F_i{(\mathbf{v})}\leq F_i(\mathbf{w)}+(\mathbf{v}-\mathbf{w})^T\nabla F_i(\mathbf{w})+\frac L2||\mathbf{v}-\mathbf{w}||^2$
\end{Assumption}

\begin{Assumption}[]
	The expected squared norm of stochastic gradients for each vehicle $k$ is upper-bounded: $\mathbb{E}\left\|\nabla f(\mathbf{w}^{r,e}_{k})\right\|^{2}\leq~G^{2}$.
\end{Assumption}

\begin{Assumption}[]
 The variance of mini-batch gradients is upper-bounded:   $\left\|g(\mathbf{w})-\nabla f(\mathbf{w})\right\|^2\leq\sigma^2.$
\end{Assumption}

\begin{Assumption}
	The divergence between local and global loss functions is bounded: $\frac{1}{K}\sum_{k=1}^K\left\|\nabla f(\mathbf{w})-\nabla F(\mathbf{w})\right\|^{2}\leq\epsilon_{g}^{2}.$
\end{Assumption}

%Here we use short expression of overall epoch $t=rE$ for superscript ${r}$ and $t=rE+e$ for superscript ${r,e}$. We define virtual global model as
%\begin{equation}\label{vir}
%	\bar{\mathcal{w}}^{t+1}=\bar{\mathcal{w}}^t-\eta\sum_{k}\mathbbm{1}_k^{t}(\frac{g_k^{t}}{\sum_{k} {\mathbbm{1}_k^{t}}}).
%\end{equation} 

\begin{lemma}
	The divergence of the local model and virtual global model is bounded: $\mathbb{E}\left[\sum_{k=1}^K\left\|\overline{\mathbf{w}}_t-\mathbf{w}_k^t\right\|^2\right] \leq 4K\eta_t^2(E-~1)^2G^2$.
\end{lemma}

%$L$ training section
%model as $\mathbf{w}^t$ and $\mathbf{w}^t_k$ as global model and local model for vehicle at

Based on the above Assumptions 1-4 and Lemma 1, we derive the following theorem to prove the convergence of the proposed scheme.  
%20240710暂时写完收敛性证明，后面看多臂老虎机的实验
\begin{theorem}
	Let $\eta_{t_0} \leq 2\eta_t, \forall t-t_0\leq E-1$, then the convergence rate of proposed FL scheme satisfies 
	\begin{equation}\begin{aligned}
			&\frac1T\sum_{t=1}^T\mathbb{E}\bigg[\|\nabla F(\bar{\mathbf{w}}^t)\|^2\bigg] \leq
			\frac{1}{T}\sum_{t=1}^{T}\frac{2}{\eta p^{t} }(\mathbb{E}[F(\bar{\mathbf{w}}^0)]-F_{\mathrm{inf}})\\
			&+{2(\epsilon_{g})^2}+\eta(\delta^2+G^2)L+4\eta^2(E-1)^2G^2L^2%<4\eta^2(E-1)^2G^2L^2
			%&-\sum_{k=1}^K\alpha_{k} G^2
	\end{aligned}\end{equation}
\end{theorem}

%20240928暂时看到此，后面补那个参考文献。。。

\begin{proof}
	See Appendix A. 
\end{proof}
\noindent\textbf{Remark 1.} From the above theorem, we can get the main factor of convergence in the proposed scheme is the real-time ratio $p^{t}$ in (\ref{p}). Considering the affecting factors of $p^{t}$, the vehicle selection choice $a_k^r$ has an important impact on the ratio $p^{t}$, which also influences the convergence speed.

\section{Problem Formulation}\label{4}
%In this section, we propose the optimization problem based on convergence analysis results to minimize the training loss and delay. 

To minimize the training loss and training delay, we design the utility function which combines the proportion of receiving models and normalized training delay as $$\Phi^{r}(a_k^r)=[\alpha p^{r}(a_k^{r})-(1-\alpha) \frac{T(a_k^r)-T_{\text{min}}}{T_{\text{max}}-T_{\text{min}}}],$$where $\mathbf{a}^r=[a_1^r,a_2^r,\ldots,a_K^r]$ is the vehicle selection strategy for all vehicles, $\alpha$ is the parameter between 0 and 1 to balance the impact of training loss and delay, and $T_{\text{min}}, T_{\text{max}} $ are minimum and maximum of round duration time, respectively.  Then we formulate the optimization problem to maximize the sum of utility functions as 
\begin{subequations}
		\begin{align}
			\nonumber\mathbf{P1}:\operatorname*{max}_{\mathbf{a}^{r}}& \sum_k\Phi^{(r)}(a_k^r)\\  %\label{opti1}\\
			\mathbf{s.t}.&  B/\sum_k a_k^r \geq B_{\text{min}},\forall r\in\mathcal{R}, k\in \mathcal{K}  \label{opti2} \\
			&a_k^r\in{0,1}, \forall  r\in\mathcal{R}, k\in\mathcal{K} \label{opti3}\\
			&\sum_k a_k^r=K_0,\forall r\in\mathcal{R}\label{opti4} 
		\end{align}
\end{subequations}
 The inequality in (\ref{opti2}) means the lower bound of bandwidth for each vehicle should be guaranteed. The expression in (\ref{opti3}) indicates the feasibility condition of vehicle selection, and the equality in (\ref{opti4}) gives the number of initially selected vehicles.

The problem P1 is difficult to solve directly because the expression of probability $p^{r}(a_k^r)$ needs future location and velocity information of selected vehicles, which are challenging to get before vehicle selection and training. We design the MAB-based vehicle selection algorithm to solve the optimization problem.

%\mathbf{P1}:\operatorname*{min}_{\{a_{k}^{r}\},\{b_{k}^{r}\}}& \sum_{r=1}^{R}\frac{\max_{k}T(a_{k}^{r},b_{k}^{r})}{N(a_{k}^{r})}  \\
%\mathbf{s.t}.& b_k^r B\geq B_{\text{min}},\forall r\in[R], k\in [K]  \\%& \sum_{k=1}^{K}a_{k}^{r}b_k^r B\leq B,\forall r\in[R]  \\
%&\frac{\beta}{P_{k}}+\gamma h(a_k^{r})<\frac{\epsilon-\zeta}{R},\forall r\in[R], k\in [K]   \\
%&\epsilon>\zeta>0

%the arrival number of vehicles at the beginning of training section follows Poisson distribution  
%$$g(K=k\lambda)=e^{-\lambda T}\frac{(\lambda T)^k}{k!}.$$ 

%The distribution of vehicle speed and arrival is independent of each other.

%\begin{equation}
	%\begin{aligned}
		%&\mathbb{E}[F(\mathbf{w}^{r+1})-F(\mathbf{w}^r)]\leq-\mathbb{E}\langle\nabla F(\mathbf{w}^r),{\eta}_k^{r}G^{r}\rangle+\frac{L}{2}\mathbb{E}\|{\eta}^{r}G^{r}\|^{2} \\
		%&\leq-\sum_{k=1}^{K}\frac{\eta_{k}^r}{2K}(\mathbb{E}\|\nabla F(\mathbf{w}^r)\|^{2}+\mathbb{E}\|{G}^{r}\|^{2}-\mathbb{E}\|\nabla F(\mathbf{w}^r)-{G}^{r}\|^{2})\\ &+\frac{L}{2}\mathbb{E}\|{\eta}^r{G}^{r}\|^{2} \\
		%&\leq\mathbb{E}\sum_{k=1}^K\frac{\eta_k^r}{2K}\|\nabla_kF(\mathbf{w}^r)-g_k^r(\mathbf{w}_k^r)\|^2& \\
		%&-\sum_{k=1}^{K}\frac{\eta_{k}^r}{2K}\mathbb{E}\|\nabla F(\mathbf{w}^r)\|^{2}-\mathbb{E}\sum_{k=1}^{K}\frac{\eta_{k}^r-L\eta_{k}^{2}}{2K}\|{G}^{r}\|^{2}.
	%\end{aligned}
%\end{equation}
%20240708这里不太好加p_r=K_0/K=\sum p_0？单个车的移动概率，怎么跟多个车的概率联系在一起！！！或者转换一下？多个车概率大于单个车？关键是这个加上E的计算。。。直接变成我的概率，求和的话。。。

\section{Proposed Algorithm}\label{5}
In this section, we provide the solution to the problem P1. The proposed vehicle selection problem in P1 can be formulated as an MAB problem \cite{Jiang_2024,Mab1,Liu2023}. The vehicles with the larger estimated accumulated utility function are selected with the upper confidence bound (UCB) policy.

In the proposed MAB vehicle selection algorithm, the exploitation-exploration trade-off is considered with a larger utility function as exploitation and the diversity of participated vehicles as exploration. The details of exploitation and exploration are shown below.

\begin{algorithm}%需要修改S，并阅读V的表示方法。
	\caption{MAB-based vehicle selection algorithm.}
	%20231103按师兄意见，暂时考虑所有车都可以找到区间内对应的D2D车辆。。。且不需要考虑单跳的数量....
	\begin{algorithmic}[1]\label{al}
		%\STATE{\color{red}tbdtbd, too short???}	
		\renewcommand{\algorithmicrequire}{\textbf{Input:}}
		\renewcommand{\algorithmicensure}{\textbf{Output:}}	
		\Require{Location of vehicles $\{x_k^r\}$, number of vehicles selected~$M^r$.}
		\Ensure{The set of selected vehicles $\mathcal{S}^r$.}
		\For{$r\leq R$}
		\If{$r=0$}
		\State{Select $K_0$ vehicles randomly as $\mathcal{S}^{0}$ in $\mathcal{K}$.}
		\Else
		\State{BS calculates the UCB score in (\ref{ucb}) and updates score list $\mathcal{U}[k]=U_k(a^r_k)$ based on $\mathcal{S}^{r-1}$.}
		\State{BS generates vehicle set} \Statex{$\mathcal{S}^{r}=$$\{K_0 \text{ vehicles with largest values of } U_k \text{ in }\mathcal{U}\}.$}\nonumber
		\EndIf
		\State{BS distributes global model to vehicles in $\mathcal{S}^{r}$.}
		\EndFor
	\end{algorithmic}
\end{algorithm}

\subsubsection{Exploitation}
We record the times of vehicle $k$ been selected before round $r$ as $$ M^{(r)}(\lambda,a_k^r)=\sum_{\tau=1}^{r} \lambda^{r-\tau}\mathbbm{1}(a_k^r=1), $$ where $\lambda$ denotes as discount factor between 0 and 1 to measure the importance of recent choices. Then we can get the discounted empirical average as 
\begin{equation}
	\bar{\Phi}^{r}(\lambda,a_k^r)=\frac{\sum_{\tau=1}^{r} \lambda^{r-\tau}\mathbbm{1}(a_k^r=1)\Phi^{(r)}(a_k^r)}{M^{r}(\lambda,a_k^r)}
\end{equation}
which gives more weight to recent objective functions.

\subsubsection{Exploration}
We give the function of the UCB index of the proposed scheme as 
\begin{equation}
	c_k(\lambda, a_k^r)=\sqrt{\frac{2\log n(r,\lambda)}{M^{r}(\lambda,a_k^r)}}, 
\end{equation}
where $n(r,\lambda)$ records the total number of all the vehicles selected before round $r$. This UCB index function facilitates the selection and exploration of alternative vehicles in other zones. At the beginning of the selection, if vehicle $k$ is not chosen for a while, then $n(r,\lambda)$ will increase with the training and $ M^{r}(\lambda,a_k^r)$ remains the same. In subsequent training rounds, vehicle $k$ will be chosen frequently with a larger UCB index. When the vehicle $k$ is selected for many rounds, the parameter $ M^{r}(\lambda,a_k^r)$ will increase bringing in a decreasing UCB index.

Based on the above formulations, the upper UCB index in MAVFL scheme is denoted as 
\begin{equation}\label{ucb}
	U_k(a_k^r)=\bar{\Phi}^{r}(\lambda,a_k^r)+c_k(\lambda, a_k^r). 
\end{equation}
During the training process for the largest $K_0$ values of $U_k$, the vehicle selection choices $a_k^r$ will be set as 1. All the selection choices are learned from past rounds of UCB scores. The details of the vehicle selection process are described in Algorithm. 1.

%202401008该到此，这部分大改。。。
\section{Simulation Results}\label{6}
%Time frame $T$. TBD V2X communication...
%In our proposed scheme, vehicles will drive into a covering segment of server. The channel power gain between vehicle $k$ and BS is shown as  
%\addtolength{\topmargin{0.06in}}

%In this section, we propose the simulation results of MAVFL scheme with MAB-based vehicle selection algorithms. We show the figures and table of the relation between accuracy and delay of MAVFL scheme under different vehicle selection algorithms. 

\subsection{Simulation Setup}
In the simulation, we consider a flow of vehicles driving through a straight road with a length of 1,000 m. The mobility pattern of the vehicle is the intelligent driver model, and the base station is situated in the middle of the road with a height of 25 m. For vehicle ML tasks, we consider image classification tasks with the CIFAR-10 dataset and GTSRB dataset, respectively. For the CIFAR-10 dataset, we utilize the ResNet-18 model, and for the GTSRB dataset, we utilize the LeNet model. Each vehicle contains 600 samples with independent identically distributed. Other relevant parameters used in this simulation are listed in Table \ref{para}.

\begin{table}[htbp]
	\caption{Simulation Parameters.\label{para}}\small
	\centering
	\begin{tabular}{|c|c|c|}	
		\hline
		\textbf{Parameter} & \textbf{Description} & \textbf{Value} \\ \hline
		$\eta$\ &Learning rate&0.01     \\ \hline
		$B_a$\ &Batch size&32\\ \hline
		%$L$\ &Length of road segment&1,000 m \\ \hline
		$Z$\ &The number of zones&20 \\ \hline
		%$H$\ &BS height&25 m \\ \hline
		$G_v$\ &BS antenna gain& 6 dBi \\ \hline
		$l_p$\ &Pass loss model& 128.1+37.6$\log_{10} d$ \\ \hline
		$B$\ &Bandwidth&3 MHz  \\ \hline
		$P_n$\ &Noise power&-114 dBm \\ \hline
		$v$\ &vehicle velocity&\{60, 80\} km/h \\ \hline
		%$L$\ &Length of covering segment&1000 m \\ \hline
		%$K_0$\ &The number of selected CV&3 \\ \hline
		$f_k$\ &GPU frequency&1.3 Ghz \\ \hline
		$\alpha$\ &Weight parameter 
		&0.6 \\ \hline
	\end{tabular}
\end{table}

%We compare the performance between our scheme and  
The baselines are as follows:

\begin{itemize} %20240816改名字
	\item \textbf{Communication-based selection (CBS)}: The vehicles are selected with the nearest distance from the BS with better communication conditions in each round. 
	\item \textbf{Remain-time based selection (RBS):} The vehicles are selected with the longest remaining time in the covering segment in each round. 
	\item \textbf{Random}: The vehicles are selected randomly through the covering segment. 
\end{itemize}

%One is "Random" which means in each round BS selects vehicles within covering segment randomly, and another is "Heuristic" as we always select vehicles within zone nearest to BS. 
\begin{table*}[htbp]
	\caption{Performance comparison on CIFAR-10 and GTSRB datasets.}
	\centering
	\label{table_1}
	\setlength{\tabcolsep}{3.5pt}
	\begin{tabular}{lcccc}
		\toprule
		\multicolumn{1}{c}{} & \multicolumn{2}{c}{\textbf{CIFAR-10}} & \multicolumn{2}{c}{\textbf{GTSRB}} \\
		\cmidrule(r){2-3} \cmidrule(r){4-5}
		\multicolumn{1}{c}{} & \multicolumn{2}{c}{Training delay to 75\% accuracy} & \multicolumn{2}{c}{Training delay to 90\% accuracy} \\
		\cmidrule(r){2-3} \cmidrule(r){4-5}
		\multicolumn{1}{c}{Method} & Training Delay for 60 km/h (sec.) & Training Delay for 80 km/h (sec.) & Training Delay for 60 km/h (sec.) & Training Delay for 80 km/h (sec) \\
		\cmidrule(r){1-1} \cmidrule(r){2-3} \cmidrule(r){4-5}
		Proposed & 2326.27 & 2357.96 & 320.24 & 339.66 \\
		CBS & 2812.62($\times$ 1.21) & 2558.51($\times$ 1.08) & 457.67($\times$ 1.42) & 492.72($\times$ 1.45) \\
		RBS & 2613.61($\times$ 1.12) & 3890.03($\times$ 1.65) & 398.5 ($\times$ 1.24) & 384.03 ($\times$ 1.13) \\
		Random & 2974.8($\times$ 1.27) & N/A & 666.14 ($\times$ 2.08) & 1298.26 ($\times$ 3.82) \\
		\bottomrule
	\end{tabular}
\end{table*}

\subsection{Simulation Results}
In Fig. {\ref{fig_1}}, we compare the training accuracy for the above vehicle selection algorithms with different velocities as 60 km/h and 80~km/h with the CIFAR-10 dataset. From Fig. \ref{fig_1_1} and \ref{fig_1_2}, we observe that the convergence speed of the proposed vehicles selection algorithm with MAB is faster, and the overall time of the proposed scheme can be substantially reduced compared to two heuristic and one random algorithms. Moreover, the reduction of the overall time is decreased on average with higher velocity because the number of vehicles is less with larger distance between vehicles. 

{Figure 3 shows the performance of different algorithms on the GTSRB dataset. We can observe that the convergence speed of the proposed vehicle selection algorithm with MAB is faster than the three baseline algorithms. We note that the improved accuracy in the GTSRB dataset is smaller than in the CIFAR dataset.  It is also worth noting that the convergence time required for GTSRB is shorter than CIFAR-10. The reason for these two situations is that the GTSRB dataset is easier to train with fewer training samples compared with CIFAR-10.}

{Table \ref{table_1} shows the results of the delay in target accuracy performance of the proposed scheme with two datasets. Especially, for the CIFAR-10 dataset, we compare the delay of all vehicle selection methods reaching 75\% accuracy, and the 90\% accuracy is considered for the GTSRB dataset. For the CIFAR-10 dataset, compared with CBS and RBS selection algorithms, the MAB algorithm is nearly 23\% faster for 60 km/h, and approximately 50\% faster for 80 km/h. In particular, the Random algorithm can not reach 80\% accuracy within the given deadline. Meanwhile, the proposed scheme is approximately 32\% faster than CBS and RBS methods for the GTSRB dataset.}

\begin{figure}[ht]
	\centering
	\subfloat[CIFAR-10 with 60 km/h\label{fig_1_1}]{
		\includegraphics[width=0.4\textwidth]{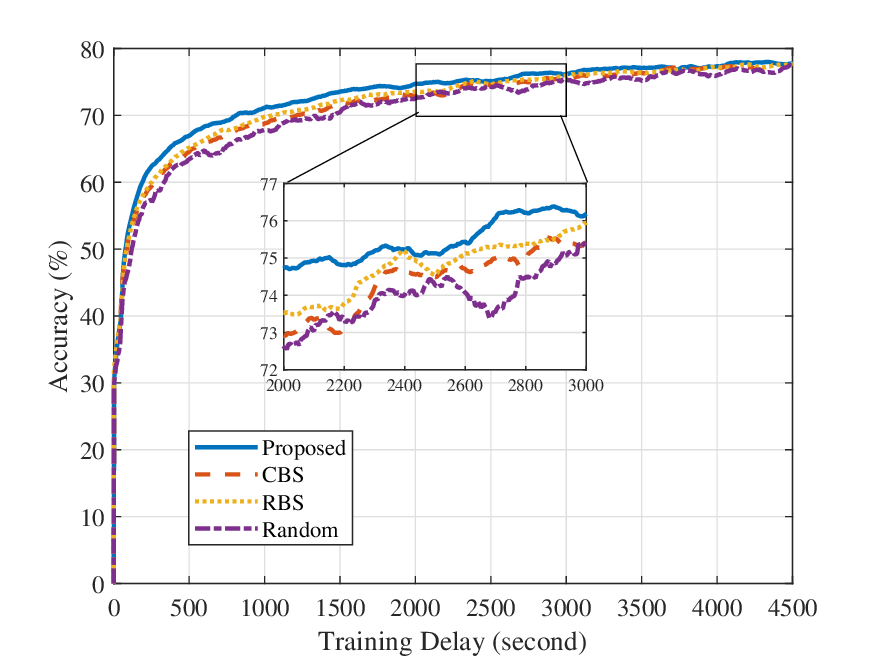}}
	\ 
	%\quad
	\subfloat[CIFAR-10 with 80 km/h\label{fig_1_2}]{
		\includegraphics[width=0.4\textwidth]{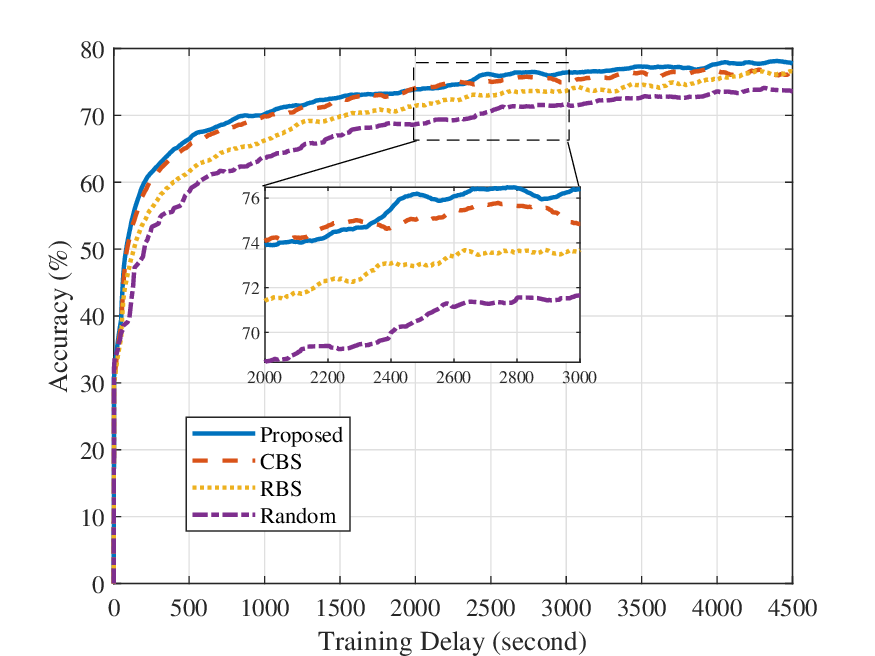}
	}
	\caption{Training performance of different vehicle selection algorithms for CIFAR-10\label{fig_1}}
\end{figure}

\begin{figure}[ht]
	\centering
	\subfloat[GTSRB with 60 km/h\label{fig_2_1}]{
		\includegraphics[width=0.4\textwidth]{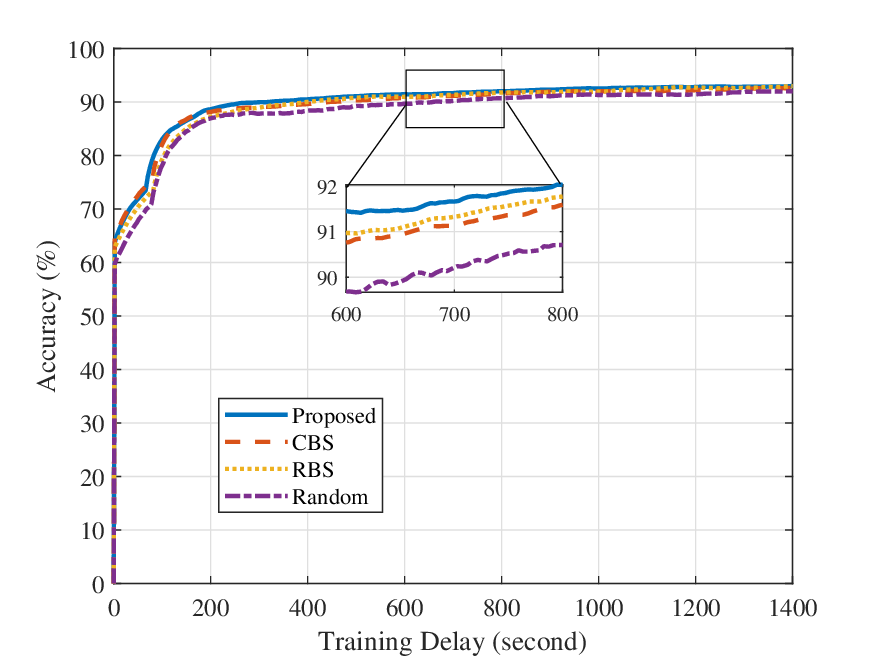} }
	\
	%\quad
	% 百度里有个方法说这里要空格，但我不空格也是上下排版的，疑惑
	\subfloat[GTSRB with 80 km/h\label{fig_2_2}]{
		\includegraphics[width=0.4\textwidth]{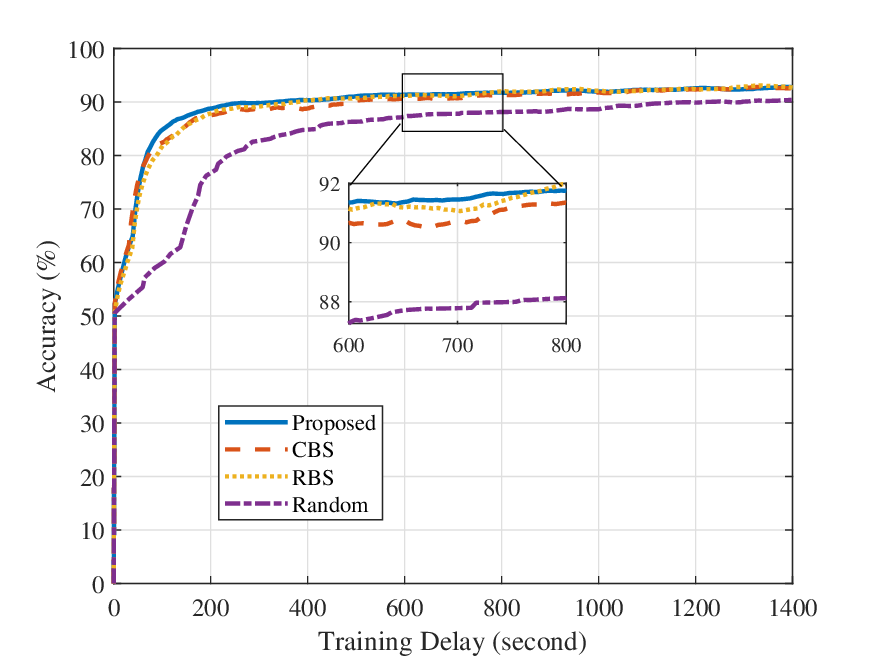}
	}
	\caption{Training performance of different vehicle selection algorithms for GTSRB\label{fig_2}}
\end{figure}

%In Table. \ref{table_1}, we show details of the training delay for different vehicle methods. For CIFAR-10 dataset, considering delay achieving 80\% accuracy for 60 km/h, our proposed MAB-based selection algorithm speeds up 16.6..\% with GSBS, 28.5..\% with GSRT, and 37.5..\% with RS. For GTSRB dataset with 60 Km/h, our proposed MAB-based selection algorithm  speeds up 19.64..\% comparing with GSRT, 30\% with GSBS, and more than 43.02 \% with RS. Based on above results, we can get our proposed scheme outperforms other selection algorithms with less delay and faster convergence.

% \usepackage{multirow}

\section{Conclusion}\label{7}
In this paper, we have designed an MAVFL scheme and proposed a real-time ratio to reflect the successful training participation rate. We also analyzed the impact of the proposed ratio on convergence results. We have formulated an optimization problem to decrease training delay by selecting suitable vehicles through the MAB-based vehicle selection algorithm. The proposed MAB-based algorithm provides a new feasible solution for real-time vehicle selection. {In future work, we will explore the proposed scheme for vehicles with computing and data heterogeneity.}
\section*{Acknowledgment}
%This work was supported in part by the Peng Cheng Laboratory Major Key Project under Grants PCL2023AS1-5 and PCL2021A09-B2, and in part by the Natural Science Foundation of China under Grant 62201311. %, and in part by the Young Elite Scientists Sponsorship Program by CAST under Grant 2023QNRC001.
This work was supported in part by the Peng Cheng Laboratory Major Key Project under Grants PCL2023AS1-5 and PCL2021A09-B2, in part by the Natural Science Foundation of China under Grant 62201311, and in part by the Young Elite Scientists Sponsorship Program by CAST under Grant 2023QNRC001.
\begin{figure*}[ht]
	\centering
	%\hrulefill
	\begin{equation}
		\begin{aligned}\label{aa}
			\mathbb{E}\bigg[F({\bar{\mathbf{w}}}^{t+1})\bigg]%= \mathbb{E}\bigg[F\left(\bar{\mathbf{w}}^{t}-\eta\sum_{k}\frac{g_k^t\mathbb{I}_k^t}{\sum_k \mathbb{I}_k^t}\right)\bigg] 
			{\operatorname*{\leq}}\mathbb{E}\bigg[F(\bar{\mathbf{w}}^{t})\bigg] 
			+\frac{\eta^2L}2\mathbb{E}\bigg[\|\sum_k(g_k^t\mathbb{I}_k^t/\sum_k \mathbb{I}_k^t)\|^2\bigg] 
			-\eta\mathbb{E}\bigg[\langle\nabla F(\bar{\mathbf{w}}^{t}),\sum_k (g_k^t\mathbb{I}_k^t/\sum_k \mathbb{I}_k^t)\rangle\bigg]
		\end{aligned}
	\end{equation}
	\begin{equation}
		\begin{aligned}\label{bb}
			\mathbb{E}& \left[\left\|\sum_k(g_k^t\mathbb{I}_k^t/\sum_k \mathbb{I}_k^t)\right\|^{2}\right]
			%20240905这里是因为最后一行的二范数，乘以的时候第一个是0.。。（车辆运动与采样的随机性。。。）
			{\operatorname*{=}}\mathbb{E}\left[\left\|   \sum_k(g_k^t\mathbb{I}_k^t/\sum_k \mathbb{I}_k^t)-p^t\sum_k\nabla f(\mathbf{w}^t_k) \right\|^2\right]-\mathbb{E}\left[\left\|p^t\sum_k\nabla f(\mathbf{w}^t_k)\right\|^2\right] \\
			&{\operatorname*{=}}\mathbb{E}\bigg[\bigg\|\sum_k(g_k^t\mathbb{I}_k^t/\sum_k \mathbb{I}_k^t)-p^t\sum_k\nabla f(\mathbf{w}^t_k)\bigg\|^2\bigg] -(p^t)^2\sum_{k}\mathbb{E}\bigg[\bigg\|\nabla f(\mathbf{w}_{k}^t)\bigg\|^2\bigg]\\
			&=\mathbb{E}\bigg[\bigg\|\sum_k(g_k^t\mathbb{I}_k^t/\sum_k \mathbb{I}_k^t)-\sum_k(\nabla f(\mathbf{w}^t_k)-(1-p^t)\nabla f(\mathbf{w}^t_k))\bigg\|^2\bigg]-(p^t)^2\sum_{k}\mathbb{E}\bigg[\bigg\|\nabla f(\mathbf{w}_{k}^t)\bigg\|^2\bigg]\\
			%20240905这里同样是因为最后一行的二范数，乘以的时候第一个是0.。。
			&{=}\sum_{k}\mathbb{E}\bigg[\|(g_k^t\mathbb{I}_{k}^t/\sum_k \mathbb{I}_k^t)-\nabla f(\mathbf{w}_{k}^t)\|^2\bigg]
			-(1-p^t)^2\sum_{k}\mathbb{E}\bigg[\|\nabla f(\mathbf{w}_{k}^t)\|^2\bigg]-(p^t)^2\sum_{k}\mathbb{E}\bigg[\bigg\|\nabla f(\mathbf{w}_{k}^t)\bigg\|^2\bigg]\\%dddd (no need 28, turn to 29)
			&\leq p^tK(\delta^2+G^2)%20240709 29与上面的合并
		\end{aligned}
		%{\noindent} \rule[-10pt]{18cm}{0.05em}
	\end{equation}
	%\hrulefill
	\begin{equation}\label{cc}
		\begin{aligned}
			-&\eta\mathbb{E}\bigg[\langle\nabla 	F(\bar{\mathbf{w}}^{t}),\sum_{k}(g_{k}^{t}\mathbb{I}_{k}^{t}/\sum_{k}\mathbb{I}_{k}^{t})\rangle\bigg]{=}-\eta 	p^t\sum_{k}\mathbb{E}\bigg[\langle\nabla F(\bar{\mathbf{w}}^{t}),\nabla f(\mathbf{w}^t_k)\rangle\bigg]\\&{=}-\frac{\eta p^{t}}{2}\sum_{k}\bigg(\mathbb{E}\bigg[\|\nabla F(\bar{\mathbf{w}}^{t})\|^{2}\bigg]+\mathbb{E}\bigg[\|\nabla f(\mathbf{w}_{k}^{t})\|^{2}\bigg]-\mathbb{E}\bigg[\|\nabla F(\bar{\mathbf{w}}^{t})-\nabla f(\mathbf{w}_{k}^{t})\|^{2}\bigg]\bigg)\\
			&{=}-\frac{\eta p^{t}}{2}\sum_{k}\bigg(\mathbb{E}\bigg[\|\nabla F(\bar{\mathbf{w}}^{t})\|^{2}\bigg]+\mathbb{E}\bigg[\|\nabla f(\mathbf{w}_{k}^{t})\|^{2}\bigg]-\mathbb{E}\bigg[\|\nabla F(\bar{\mathbf{w}}^{t})-\nabla F(\mathbf{w}_k^t)+\nabla F(\mathbf{w}_k^t)-\nabla f(\mathbf{w}_{k}^t)\|^{2}\bigg]\bigg)\\
			&\leq -\frac{\eta p^{t}}{2}\sum_{k}\bigg(\mathbb{E}\bigg[\|\nabla F(\bar{\mathbf{w}}^{t})\|^{2}\bigg]+\mathbb{E}\bigg[\|\nabla f(\mathbf{w}_{k}^{t})\|^{2}\bigg]-2\mathbb{E}\bigg[\|\nabla F(\bar{\mathbf{w}}^{t})-\nabla F(\mathbf{w}_k^t)\|^{2}+\mathbb{E}\|\nabla F({\mathbf{w}}^{t}_k)-\nabla f(\mathbf{w}_k^t)\|^{2}\bigg]\bigg)\\
			%&\leq -\frac{\eta p^{(t)}}{2}\sum_{k}\bigg(\mathbb{E}\bigg[\|\nabla F(\bar{\mathbf{w}}^{(t)})\|^{2}\bigg]+G^2-2\bigg[L^2\mathbb{E}\| \bar{\mathbf{w}}^{(t)}- \mathbf{w}_k^t\|^{2}+\epsilon_{g}^2\bigg]\bigg)\\ %20240710两个平方和见(32)，。。。20240710补...
			%&\leq -\frac{\eta p^{t}}{2}\sum_k(\mathbb{E}\|\nabla F(\bar{w}^{(t)})\|^{2}+4p^{(t)}\eta^3(E-1)^2G^2L^2+{\eta p^{(t)} \epsilon_g^2})
			&\leq -\frac{\eta p^{t}}{2}\sum_k(\mathbb{E}\|\nabla F(\bar{\mathbf{w}}^{t})\|^{2}-4\eta^2(E-1)^2G^2L^2-2{\epsilon_g^2})
			%20240710补完结论，下面的换了。。。。
		\end{aligned}
	\end{equation}
	\hrulefill
\end{figure*}

\appendices
\section{Proof of Theorem 1}
%\subsection{ss}
%Based on the definition of virtual global model in (\ref{vir}), we can get the results in (\ref{vv}).
%\begin{equation}\label{aa}
%	\begin{aligned}
%		\mathbb{E}\bigg[F({\bar{w}}^{t+1})\bigg]& = \mathbb{E}\bigg[F\left(\bar{\mathbf{w}}^{t}-\eta\sum_{k}(g_k^t\mathbb{I}_k^t/\sum_k \mathbb{I}_k^t)\right)\bigg] \\
%		&{\operatorname*{\leq}}\mathbb{E}\bigg[F(\bar{\mathbf{w}}^{t})\bigg] \\
%		&+\frac{\eta^2L}2\mathbb{E}\bigg[\|\sum_k(g_k^t\mathbb{I}_k^t/\sum_k \mathbb{I}_k^t)\|^2\bigg] \\
%		&-\eta\mathbb{E}\bigg[\langle\nabla F(\bar{\mathbf{w}}^{t}),\sum_k (g_k^t\mathbb{I}_k^t/\sum_k \mathbb{I}_k^t)\rangle\bigg],
%	\end{aligned}
%\end{equation}
Based on the definition of the virtual global model as $\bar{\mathbf{w}}^{t+1}=\bar{\mathbf{w}}^t-\eta\sum_{k}\mathbbm{1}_k^{t}(\frac{g_k^{t}}{\sum_{k} {\mathbbm{1}_k^{t}}})$, we can get the results in (\ref{aa}). Considering the independence of vehicle mobility and local computing, we can get $\mathbb{E}[\sum_k (g_k^t\mathbb{I}_k^t/\sum_k \mathbb{I}_k^t)]=p^t\sum_k\nabla f(\mathbf{w}^t_k),$ which means the expected proportion of the number of model updates obtained by server is $p^{t}$ in round $t$ .

For the second term in (\ref{aa}), we can get the results in (\ref{bb}) at the top of the page.
For the third term in (\ref{aa}), we can get the results in (\ref{cc}). Then we substitute the second and third parts in (\ref{aa}) with (\ref{bb}) and (\ref{cc}) and average over all global rounds to finish the proof.

%We substitute the second and third part in (\ref{aa}) with (\ref{bb}) and (\ref{cc}) and averaging over all global rounds to finish the proof.
%{\color{red} The (a)(b)) c may need to be removed}

\bibliographystyle{IEEEtran}
\bibliography{paper-re-workshop-ver}

\end{document}